\documentclass[11pt]{article}

\usepackage{microtype}
\usepackage{graphicx}
\usepackage[round]{natbib}

\usepackage{booktabs} 

\usepackage[
            CJKbookmarks=true,
            bookmarksnumbered=true,
            bookmarksopen=true,
            colorlinks=true,
            citecolor=red,
            linkcolor=blue,
            anchorcolor=red,
            urlcolor=blue
            ]{hyperref}

  \renewcommand{\footnoterule}{%
  \kern -3pt
  \hrule width \textwidth height 1pt
  \kern 2pt
}

\usepackage{amsmath,amssymb,amsthm}
\usepackage{prettyref,xspace,graphicx}
\usepackage{bbm,tikz}
\usepackage{balance}

\newrefformat{cond}{Condition~\ref{#1}}
\newrefformat{eq}{(\ref{#1})}
\newrefformat{thm}{Theorem~\ref{#1}}
\newrefformat{th}{Theorem~\ref{#1}}
\newrefformat{chap}{Chapter~\ref{#1}}
\newrefformat{sec}{Section~\ref{#1}}
\newrefformat{algo}{Algorithm~\ref{#1}}
\newrefformat{fig}{Figure~\ref{#1}}
\newrefformat{tab}{Table~\ref{#1}}
\newrefformat{rmk}{Remark~\ref{#1}}
\newrefformat{clm}{Claim~\ref{#1}}
\newrefformat{def}{Definition~\ref{#1}}
\newrefformat{cor}{Corollary~\ref{#1}}
\newrefformat{lmm}{Lemma~\ref{#1}}
\newrefformat{prop}{Proposition~\ref{#1}}
\newrefformat{pr}{Proposition~\ref{#1}}
\newrefformat{app}{Appendix~\ref{#1}}
\newrefformat{prob}{Problem~\ref{#1}}
\newrefformat{ques}{Question~\ref{#1}}
\newrefformat{note}{Note~\ref{#1}}
\newrefformat{assump}{Assumption~\ref{#1}}
\newrefformat{issue}{Issue ~\ref{#1}}
\newrefformat{fix}{Fix ~\ref{#1}}

\newcommand{\reals}{\mathbb{R}}
\newcommand{\naturals}{\mathbb{N}}

\newcommand{\Expect}{\mathbb{E}}

\newcommand{\prob}[1]{\mathbb{P}\left[#1\right]}

\usepackage{lmodern}

\newcommand{\inner}[2]{\langle {#1}, {#2}  \rangle}

\newcommand{\calF}{\mathcal{F}}

\newcommand{\calN}{\mathcal{N}}

\newcommand{\calS}{\mathcal{S}}

\newcommand{\norm}[1]{\left \|#1\right \|}

\newcommand{\define}{\triangleq}

\newcommand{\pth}[1]{\left( #1 \right)}

\newcommand{\sth}[1]{\left\{ #1 \right\}}

\newcommand{\ie}{i.e.\xspace}

\newtheorem{theorem}{Theorem}

\theoremstyle{definition}
\newtheorem{example}{Example}

\usepackage[utf8]{inputenc} 
\usepackage[T1]{fontenc}    
\usepackage{hyperref}       
\usepackage{url}            
\usepackage{booktabs}       
\usepackage{amsfonts}       
\usepackage{nicefrac}       
\usepackage{microtype}      

\usepackage{amsmath,amsfonts,amsthm,amssymb}
\usepackage{graphicx,float,wrapfig}
\usepackage{multirow}

\newtheorem{propo}{Proposition}[section]
\newtheorem{lemma}[propo]{Lemma}

\newtheorem{assumption}{Assumption}
\newtheorem{remark}{Remark}

\usepackage{algorithm}
\usepackage{algorithmic}

\begin{document}






\title{Learning One-hidden-layer Neural Networks\\
 under General Input Distributions}

\author{
Weihao Gao\thanks{Equal Contribution},  Ashok Vardhan Makkuva$^{*}$, Sewoong Oh, and Pramod Viswanath\thanks{Emails: \texttt{\{wgao9,makkuva2,swoh,pramodv\}@illinois.edu}.}\\
	University of Illinois at Urbana-Champaign
}

\date{}

\maketitle

\begin{abstract}

Significant advances have been made recently on training neural networks,
where the main challenge is in solving an
optimization problem with abundant critical points.
However, existing approaches to address this issue crucially rely
on a restrictive assumption: the training data is drawn from a Gaussian distribution.
In this paper, we provide a novel unified framework to
design loss functions with desirable landscape properties
for a wide range of general  input distributions. On these loss
functions, remarkably, stochastic gradient descent
theoretically recovers the true parameters with \emph{global} initializations
and empirically outperforms the existing approaches. Our loss function
design bridges the notion of score functions with the topic
of neural network optimization.
Central to our approach is
the task of estimating the score function from samples,
which is of basic and independent interest to theoretical statistics.
Traditional estimation methods (example: kernel based) fail right at the outset; we bring statistical methods of local likelihood to design a novel
estimator of  score functions, that provably adapts to the local geometry of the unknown  density.

%
\end{abstract}

\section{Introduction}
\label{sec:intro}


Neural networks have made significant impacts over the past decade, 
thanks to their successful applications across multiple domains including 
computer vision, natural language processing, and robotics. 
This success partly owes to the mysterious phenomenon that 
(stochastic) gradient method applied to 
highly non-convex loss functions converges to a model parameter that achieves high test accuracy. 
We are in a dire need of theoretical understanding of such phenomenon, 
in order to guide the design of next generation neural networks and training methods. 
Significant recent progresses have been made, 
by asking a simpler question: 
can we efficiently learn a neural network model, 
when there is a ground truth neural network that generated the data? 

Suppose the data $(x,y)$ is generated by sampling $x$ from 
an unknown distribution $f_X(x)$ and $y$ is generated by 
passing $x$ through an unknown neural network and adding some simple noise. 
Even if we train neural networks on this ``teacher network'', it is known to be a hard problem
without further assumptions \citep{brutzkus2017globally}.   
Significant effort has been on 
designing new approaches to learn 
simple neural networks (such as one-hidden-layer neural network) 
on data from simple distributions (such as Gaussian) \citep{tian2017analytical,ge2017learning}. 
This is followed by analyses on increasingly more complex architectures \citep{brutzkus2017globally,li2017convergence}.
However, the analysis techniques critically depend on the Gaussian input assumption, 
and further the proposed algorithms are tailored specifically to Gaussian inputs. 
In this paper, we provide a  unified approach to design loss functions that 
provably learn the true model  for a wide range of input distributions with smooth densities. 

We consider a scenario where the data is generated from a 
one-hidden-layer neural network 
\begin{eqnarray}
	y \; =\; \sum_{i=1}^k w_i^\ast \, g(\inner{a_i^\ast}{x}) \,+\, \eta \;.
	\label{eq:model0}
\end{eqnarray}
where the true parameters are $w_i^\ast \in \reals$ and  $a_i^\ast \in \reals^d$, 
and  $\eta$ is a zero-mean noise independent of $x$, with some non-linear activation function $g:\reals\to\reals$. 
It has been widely known that
first order methods on the $\ell_2$-loss get stuck in bad local minima,
even for this simple one-hidden-layer neural networks \citep{livni2014computational}.
If the input $x$ is coming from a Gaussian distribution,  
\cite{ge2017learning} proposes a new loss function $G(\cdot)$ with a carefully designed landscape such that 
Stochastic Gradient Descent (SGD) provably converges to the true parameters. 
However, the proposed novel loss function is specifically designed for Gaussian inputs, 
and gets stuck at bad local minima when applied to general non-Gaussian distributions. 
We showcase this in \prettyref{fig:exponential}.
Designing the optimization landscape for
general input distributions is a
practically important and technically challenging problem,
 as acknowledged in \cite{ge2017learning}
 and many existing works in the literature~\citep{brutzkus2017globally,tian2017analytical,li2017convergence}.

Our goal is to strictly generalize the approach of \cite{ge2017learning} and 
construct a  loss function $L(\cdot)$ with a good landscape such that SGD 
recovers the true parameters with \emph{global} initializations.
The main challenge is in  estimating the {\em score function} defined as 
a functional of the probability density function $f(x)$ of the input data $x$: 
\begin{eqnarray}
    \mathcal{S}_m(x) \;\; \define \;\; \frac{\nabla^{(m)} f_X(x)}{f_X(x)} \;,
    \label{eq:defscore}
\end{eqnarray}
where $\nabla^{(m)} f_X(x)$ denotes the $m$-th order derivative 
for an $m\in{\mathbb Z}$, 
which plays a crucial role in the landscape design. 
We need to evaluate this score function at sample points, which 
 is extremely challenging as  it involves the {\em higher order  derivatives} of a pdf that we do not know.  
Standard non-parametric density estimation methods 
such as the Kernel Density Estimators (KDE) \citep{fukunaga1975estimation} and 
$k$-Nearest Neighbor methods ($k$NN) all fail to provide a consistent estimator, as they are tailored for density estimation.  
Existing heuristics do not have even consistency guarantees, which include 
score matching based methods~\citep{hyvarinen2005estimation,swersky2011autoencoders}, 
and de-noising auto-encoder (DAE) based algorithms \citep{janzamin2015feast}.

In this paper, we first address this fundamental question of
how to estimate the score functions  from samples
in a principled manner.
We introduce a novel 
approach to adaptively capture the
local geometry of the pdf to
design a consistent estimator for score functions.
To achieve this, we bring ideas from {\it local likelihood } methods ~\citep{loader1996local,hjort1996locally} from statistics to the context of score function estimation and also prove the convergence rate of our estimator (LLSFE), which is of independent mathematical interest.
We further introduce a new loss function for training one-hidden-layer neural networks, 
that builds upon the estimated score functions.  
We show that this provably has the desired landscape for general input distributions.

In summary, our main contributions are:

\begin{itemize}
    \item {\bf Score function estimation.} In this paper, we provide the first consistent estimator for score functions (and hence the gradients of $L(\cdot)$), 
    which play crucial roles in several recent model parameter learning problems \citep{hyvarinen2005estimation,swersky2011autoencoders,janzamin2015feast}.  
    Our provably consistent estimation of score functions, LLSFE, from samples, with local geometry adaptations, is   of independent mathematical interest.
    
    \item {\bf Optimization landscape for general distributions.}
    For a large class of input distributions, with an appropriate score transformation for the input and  appropriate tensor projection, we design a  loss function $L(\cdot)$ for one-hidden-layer neural network with good landscape properties. In particular, our result is a strict generalization of \citep{ge2017learning} which was restricted to Gaussian inputs, in both mathematical and abstract view-points.
%

%
%

\end{itemize}


\noindent {\bf Related work.} Several recent works have provided provable algorithms for training neural networks \citep{liang2018understanding,choromanska2015loss,soudry2017exponentially,goel2017learning,freeman2016topology,nguyen2017loss,boob2017theoretical}. \citep{arora2014provable} is an early work on provable learning guarantees on deep generative models for sparse weights. \cite{brutzkus2017globally,tian2017analytical} analyze one-hidden-layer neural network with Gaussian input and hidden variables with disjoint supports. \citep{li2017convergence} analyzed the convergence of one-hidden layer neural network with Gaussian input when the true weights are close to identity. \citep{andoni2014learning}, \citep{panigrahy2018convergence}, \citep{du2018power} and~\citep{soltanolkotabi2017theoretical} studied the optimization landscape of neural networks for some specific activation functions.

Tensor methods have been used to build provable algorithms for training neural networks \citep{janzamin2015beating,zhong2017recovery}. Our work is built upon~\citep{ge2017learning}, which uses a  fourth-order tensor based objective function and show good landscape properties. Most of the aforementioned works requires specific assumptions on the input distribution (example: Gaussian), while we only require generic smoothness of the underlying (unknown) density. In a recent work, \cite{ge2018learning} provided a learning algorithm using the method of moments for symmetric input distributions. However their techniques are very specific to ReLU activation and do not generalize to general activation functions which we can handle. \cite{zhang2018learning} show that gradient descent on empirical loss function based on least squares can recover the true parameters provided the parameters have a good initialization; in contrast, we use \emph{global} initializations for our algorithm.
\\

{\bf Notations.} We use $\mathcal{T}(x_1,  \dots, x_m)$ to denote the inner product for an $m$-th order tensor $\mathcal{T}$ and vectors $x_1, \dots, x_m$. We use $x_1 \otimes x_2 \otimes \dots \otimes x_m$ to denote outer product of vectors/matrices/tensors. $x^{\otimes j} = x \otimes \dots \otimes x$ denotes the $j$-th order tensor power of $x$ and $x^{\otimes 0} = 1$. 

$\|\mathcal{T}\|_{{\rm sp}} = \max_{\|u_i\|_2 \leq 1} \mathcal{T}(u_1, u_2, \dots, u_m)$ and $\|\mathcal{T}\|_{\mathcal{F}} = \sqrt{ \sum_{i_1, \dots, i_m} (\mathcal{T}_{(i_1, \dots, i_m)})^2}$ denotes the spectral norm and Frobenius norm of matrix and high-order tensor. ${\rm sym}(\mathcal{T})$ denotes the symmtrify operator of a tensor $\mathcal{T}$ as ${\rm sym}(\mathcal{T})_{(i_1, \dots, i_m)} = \frac{1}{m!} \sum_{(j_1, \dots, j_m) \in \pi(i_1, \dots, i_m)} \mathcal{T}_{(j_1, \dots, j_m)}$.

\section{Score Function Estimation}
\label{sec:estimation}

In this section, we introduce a new approach for estimating score functions defined in~\eqref{eq:defscore} 
from i.i.d.~samples from a distribution. 
As the score functions involve higher order derivatives of the pdf, 
it is critical to capture the rate of {\em changes} in the pdf.
Further, we aim to apply it to data coming from a broad range of distributions.  
Such sharp estimates for such broad class of distributions can only be achieved by 
combining the strengths of two popular approaches in density estimation: 
simple parametric density estimators and complex non-parametric density estimators. 
We bridge this gap by borrowing the techniques from 
Local Likelihood Density Estimators (LLDE) and bring them to a new light 
in order to provide the first consistent score function estimators. 

\subsection{Local Likelihood Density Estimator (LLDE)}

How do we estimate the normalized derivatives of the density? We address this question in a principled manner utilizing the notion {\it local likelihood density estimation} (LLDE) 
from non-parametric methods~\citep{loader1996local,hjort1996locally}. 
LLDE is originally designed for estimating density for distributions with complicated local geometry, 
and can be further applied to estimate functionals of density such as information entropy~\citep{gao2016breaking}. 
Inspired by the fact that LLDEs capture the local geometry of the pdf, 
we build upon the LLDE estimators to design a new estimator of the 
higher order derivatives, which is the main bottleneck in score function estimation. 

The local likelihood density estimator is 
specified by a nonnegative function $K: \mathbb{R}^d \to \mathbb{R}$ (also called a Kernel function), 
a degree $p \in \mathbb{Z}^+$ of the polynomial approximation, 
and a bandwidth $h \in \mathbb{R}^+$. 
It is the solution of a maximization of the local log-likelihood function:
\begin{eqnarray}
    \mathcal{L}_x(f) &=& \sum_{i=1}^n K\Big(\frac{X_i - x}{h} \Big) \log f(X_i)  - n \int K\Big( \frac{u-x}{h} \Big) f(u) du \;.
\end{eqnarray}
For each $x$, we maximize this function over a parametric family  of functions $f(\cdot)$, 
using  the following local polynomial approximation of $\log f(x)$: 
\begin{eqnarray}
\log f(x) &=& a_0 + a_1^T(u-x) + \frac{1}{2} (u-x)^T A_2 (u-x) + \dots + \frac{1}{p!} \mathcal{A}_p (u-x, \dots, u-x) \;,
\end{eqnarray}
parameterized by $a = (a_0, a_1, A_2, \dots, \mathcal{A}_p) \in \mathbb{R} \times \mathbb{R}^d \times \mathbb{R}^{d^2} \times \dots \times \mathbb{R}^{d^p}$. 
The local likelihood density estimate
(LLDE) at point $x$ is defined as $f(x)=e^{\widehat{a}_0}$, where 
$\widehat{a} = (\widehat{a}_0, \widehat{a}_1, \widehat{A}_2, \dots, \widehat{\mathcal{A}}_p)$
is the maximizer around a point $x$: $\widehat{a} \in \arg\max_{f} {\cal L}_x(f) $. 
The optimization problem can be solved by setting the derivatives 
$\partial \mathcal{L}_x(p)/\partial \mathcal{A}_j = 0$ for $j \in \{0, \dots, p\}$. 
The optimal solution 
$\widehat{a}$ can be obtained from solving the following equations,
\begin{eqnarray}
&&\int_{\mathbb{R}^d} \exp\{a_0 + a_1^T(u-x) + \dots + \frac{1}{p!} \mathcal{A}_p (u-x)^{\otimes p}\}  (\frac{u-x}{h})^{\otimes j} K(\frac{u-x}{h}) du \,\notag\\
&=& \frac{1}{n} \sum_{i=1}^n (\frac{X_i-x}{h})^{\otimes j} K(\frac{X_i-x}{h}) \;.
\label{eq:LLDE}
\end{eqnarray}

We build upon this idea to first introduce the score function estimator, 
and focus on the statistical aspect of this estimator.  
We discuss the computational aspect of finding the solution to this optimization in Section~\ref{sec:computational}.

\subsection{From LLDE to local likelihood score function estimator (LLSFE)}

We build upon the techniques from LLDE to design our local likelihood score function estimator (LLSFE). 
Notice that the score function $\mathcal{S}_m(x)$ satisfies the following recursive formula from \citep{janzamin2014score},
\begin{eqnarray}
    \mathcal{S}_m(x) &=& -\mathcal{S}_{m-1}(x) \otimes \nabla_x \log f(x) - \nabla_x \mathcal{S}_{m-1}(x) ,\notag
\end{eqnarray}
and $\mathcal{S}_1(x) = -\nabla \log f(x)$. 
This recursion reveals us that the score function can be represented as a polynomial function of 
the gradients of log-density $(g_1(x), G_2(x), \dots, \mathcal{G}_m(x)) = \\
	(\nabla \log f(x), \nabla^{(2)} \log f(x), \dots, \nabla^{(m)} \log f(x))$. For example, the polynomial for $\mathcal{S}_2(x)$ and $\mathcal{S}_4(x)$  are given below:
\begin{eqnarray}
\mathcal{S}_2(x) &=& g_1(x) \otimes g_1(x) + G_2(x) \,\\
\mathcal{S}_4(x) &=& g_1(x) \otimes g_1(x) \otimes g_1(x) \otimes g_1(x) + 6 \,{\rm sym} (G_2(x) \otimes g_1(x) \otimes g_1(x)) \,\notag\\
&&+\, 3 \,(G_2(x) \otimes G_2(x)) + 4 \,{\rm sym} (\mathcal{G}_3(x) \otimes g_1(x)) + \mathcal{G}_4(x)
\end{eqnarray}
More generally, the $m$-th order score function can be represented as:
\begin{eqnarray}
\mathcal{S}_m(x) &=& \sum_{\lambda \in \Lambda_m} (-1)^m  c_m(\lambda)\, {\rm sym}(\bigotimes_{j \in \lambda} \mathcal{G}_j)  \;,
\label{eq:S_poly}
\end{eqnarray}

where $\Lambda_m$ denotes the set of partitions of integer $m$ and $c_m(\lambda)$ is a positive constant depends on $m$ and the partition, for example, $\Lambda_4 = \{\{1,1,1,1\},\{2,1,1\},\{2,2\},\{3,1\},\{4\}\}$. Given the polynomial representation of a score function, the LLSFE is given by
\begin{eqnarray}
\widehat{\mathcal{S}^{(p)}_m}(x) & \triangleq & \sum_{\lambda \in \Lambda_m} (-1)^m  c_m(\lambda)\, {\rm sym}(\bigotimes_{j \in \lambda} \widehat{\mathcal{A}}_j^{(p)}) \;.
\label{eq:hat_S_poly}
\end{eqnarray}
where $\mathcal{A}_j^{(p)}$ is the LLDE estimator of $\mathcal{G}_j$ by $p$-degree polynomial approximation.

\subsection{Convergence rate of LLSFE}

As LLDE captures the local geometry of the pdf, 
LLSFE inherits this property and is able to consistently estimate the derivatives. 
This is made precise in the following theorem, where 
we provide an upper bound of the spectral norm error of the estimated $m$-th order score function. 
First, we formally state our assumptions.
\begin{assumption}
    \label{assum_1}
    \begin{enumerate}
        \item[(a)] The degree of polynomial $p$ is greater than or equal to $m$.
        \item[(b)] The gradient of log-density $\nabla^{(j)} \log f(x)$ at $x$ exists and $\|\nabla^{(j)} \log f(x)\|_{{\rm sp}} \leq C_j(x)$ for all $j \in [p+1]$.
        \item[(c)] The non-negative kernel function $K$ satisfies $ \int_{\mathbb{R}^d} |x_i|^p K(x) dx < +\infty$ for any $i \in [d]$.
        \item[(d)] Bandwidth $h$ depends on $n$ such that $h \to 0$ and $nh^{d+2m} \to \infty$ as $n \to \infty$.
    \end{enumerate}
\end{assumption}

The following theorem provides an upper bound on the  convergence rate of the proposed score function estimator. 

\begin{theorem}
\label{thm:theorem_1}
Under Assumption~\ref{assum_1}, the spectral norm error of the LLSFE $\widehat{\mathcal{S}^{(p)}_m}(x)$ defined in \eqref{eq:hat_S_poly} is upper bounded by
\begin{eqnarray}
\| \widehat{\mathcal{S}^{(p)}_m}(x) - \mathcal{S}_m(x) \|_{\rm sp} &\leq& O(d^{m/2} h^{p+1-m}) + O_p(d^{m/2} (nh^{d+2m})^{-1/2}).
\end{eqnarray}
\end{theorem}

\begin{proof}(Sketch)
Note that the estimator is derived by replacing the truth gradients of log-density $(g_1(x), G_2(x), \dots, \mathcal{G}_m(x))$ by the estimates $(\widehat{a}_0(x), \widehat{a}_1(x), \widehat{A}_2(x), \dots, \widehat{\mathcal{A}}_p(x))$. Since we assumes that $\|\mathcal{G}_j(x)\|_{\rm sp} \leq C_j$, so it suffices to upper bound the spectral norm of the error $\|\widehat{\mathcal{A}}_j^{(p)}(x) - \mathcal{G}_j(x)\|_{\rm sp}$. The following lemma provides upper bounds for each entry of $\widehat{\mathcal{A}}_j^{(p)}(x) - \mathcal{G}_j(x)$. For simplicity of notation, we fix an $x$ drop the dependency on $x$.

\begin{lemma}{\cite[Theorem 1]{loader1996local}}
\label{lem:lemma_1}
Under Assumption~\ref{assum_1}  we have
\begin{eqnarray}
\left( \widehat{\mathcal{A}}_j^{(p)} \right)_{(i_1, \dots, i_j)} - \left( \mathcal{G}_j \right)_{(i_1, \dots, i_j)} &=& O(h^{p+1-j}) + O_p((nh^{d+2j})^{-1/2}),
\end{eqnarray}
for any $j \in \{1, \dots, p\}$ and $i_1, \dots, i_j \in [d]^j$.
\end{lemma}

The spectral norm of of the error $\|\widehat{\mathcal{A}}_j^{(p)} - \mathcal{G}_j\|_{\rm sp}$ is upper bounded by the Frobenius norm. Then applying Lemma~\ref{lem:lemma_1}, we have,
\begin{eqnarray}
\|\widehat{\mathcal{A}}_j^{(p)} - \mathcal{G}_j\|_{\rm sp} &\leq& O(d^{j/2}h^{p+1-j}) + O_p(d^{j/2}(nh^{d+2j})^{-1/2}).
\end{eqnarray}

Substituting this result into the polynomial representations~\eqref{eq:S_poly} and~\eqref{eq:hat_S_poly}, we obtain the desired rate.
\end{proof}

\begin{remark}
\label{rem:remark_1}
By setting $h = n^{-1/(2p+2+d)}$, we obtain
\begin{eqnarray}
\| \widehat{\mathcal{S}^{(p)}_m}(x) - \mathcal{S}_m(x) \|_{\rm sp} &\leq&  O_p(d^{m/2} n^{-(p+1-m)/(2p+2+d)}).
\end{eqnarray}
\end{remark}

\begin{remark}
    It was shown in~\citep{stone1980optimal} that the optimal rate for estimating an entry of $\mathcal{G}_j$ is $O_p(n^{-(p+1-m)/(2p+2+d)})$. We conjecture that LLSFE is also minimax rate-optimal.
\end{remark}

\subsection{Second Degree LLSFE}
\label{sec:computational}

In the previous subsection, we proved the convergence rate of the LLSFE. However, the computational cost of LLSFE can be 
 large since numerical integration is needed to compute the integral in~\eqref{eq:LLDE}. 
 To trade off the accuracy and computational cost, we choose Gaussian kernel $K(u) \propto \exp\{\|u\|^2/2\}$ and degree $p = 2$. This makes the integration in the LHS of~\eqref{eq:LLDE} tractable and we obtain closed-form expressions for $a_0$, $a_1$ and $A_2$. Using ideas from ~\cite[Proposition 1]{gao2016breaking}, our estimators for $a_1$ and $A_2$ are:
\begin{eqnarray}
    \widehat{a}_1 &=& (\frac{M_2}{M_0} - (\frac{M_1}{M_0})(\frac{M_1}{M_0})^T)^{-1} \frac{M_1}{M_0} \;, \,\\
    \widehat{A}_2 &=& h^{-2}I_{d \times d} - (\frac{M_2}{M_0} - (\frac{M_1}{M_0})(\frac{M_1}{M_0})^T)^{-1} \;,
\end{eqnarray}
where $M_j = \sum_{i=1}^n (X_i - x)^{\otimes j} \exp\{-\frac{\|X_i - x\|^2}{2h^2}\}$ for $j \in \{0,1,2\}$.

The second degree LLSFE is derived by plugging $\widehat{a}_1$ and $\widehat{A}_2$ into~\eqref{eq:hat_S_poly}. The computational complexity of second degree LLDFE is $O(n \cdot d^2)$. In the experiments below, we use this second degree estimator and practically show that using second degree estimator as an compromise does not hurt the performance by much.

\subsection{Synthetic Simulations of LLSFE}
\label{subsec:synthetic}

In this experiment we validate the performance of LLSFE, for both Gaussian and non-Gaussian distributions. For Gaussian distribution, we choose $x \sim \mathcal{N}(0,I_d)$ and $d = 2$. The ground truth score functions are $\mathcal{S}_2 = xx^T - I_d$ and $\mathcal{S}_4 = x^{\otimes 4} - 6{\rm sym}(x \otimes x \otimes I_d) + 3 I_d \otimes I_d$. We show the spectral error $\|\widehat{\mathcal{S}}_2 - \mathcal{S}_2\|_{\rm sp}$ versus number of sample $n$ for estimation of $\mathcal{S}_2$, and the Frobenius error $\|\widehat{\mathcal{S}}_4 - \mathcal{S}_4\|_{\mathcal{F}}$ for estimation of $\mathcal{S}_4$ (since computing spectral norm of high-order tensor is NP-hard~\cite{hillar2013most}). We plot the $\{95\%, 75\%, 50\%, 25\%, 5\%\}$ percentiles of our estimation error over $10,000$ independent trials for the estimation of $\mathcal{S}_2$ and $50,000$ independent trails for the estimation of $\mathcal{S}_4$.

We can see from Figure~\ref{fig:gaussian} that all the percentiles of the estimation error decrease as $n$ increases. The log-log scale plot is closed to linear, and the average slope is $-0.5143$ for $\|\widehat{\mathcal{S}}_2 - \mathcal{S}_2\|_{\rm sp}$ and $-0.4984$ for $\|\widehat{\mathcal{S}}_4 - \mathcal{S}_4\|_{\mathcal{F}}$. This suggests that LLSFE is consistent and the error decreases at a faster rate than the theoretical upper bound in Remark~\ref{rem:remark_1}.

\begin{figure}[h]
    \centering
    \includegraphics[width=0.4\textwidth]{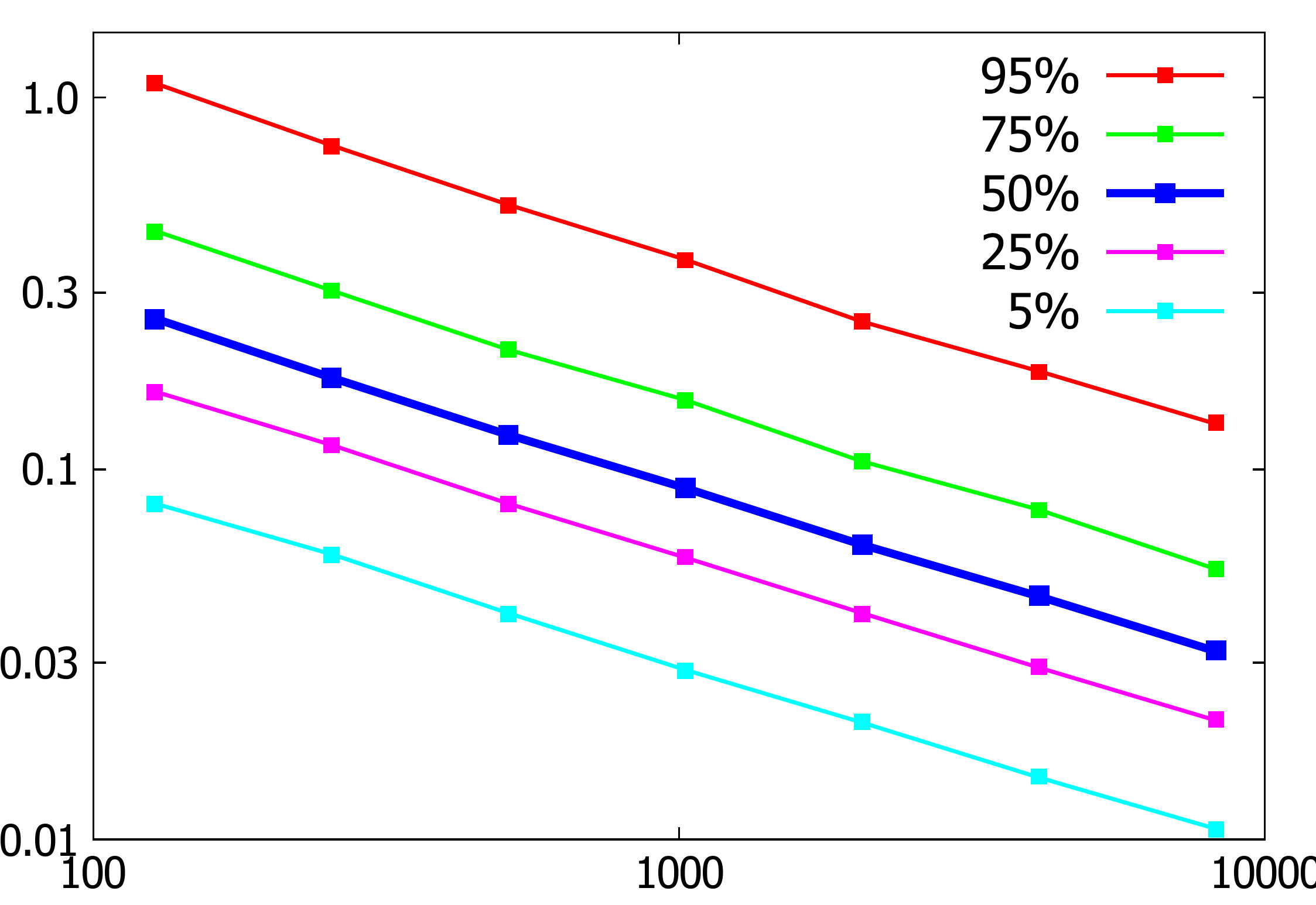}
    \put(-120,-10){sample size $n$}
    \put(-210,45){\rotatebox{90}{$\|\widehat{\mathcal{S}}_2 - \mathcal{S}_2\|_{\rm sp}$}}
    \hspace{0.25in}
    \includegraphics[width=0.4\textwidth]{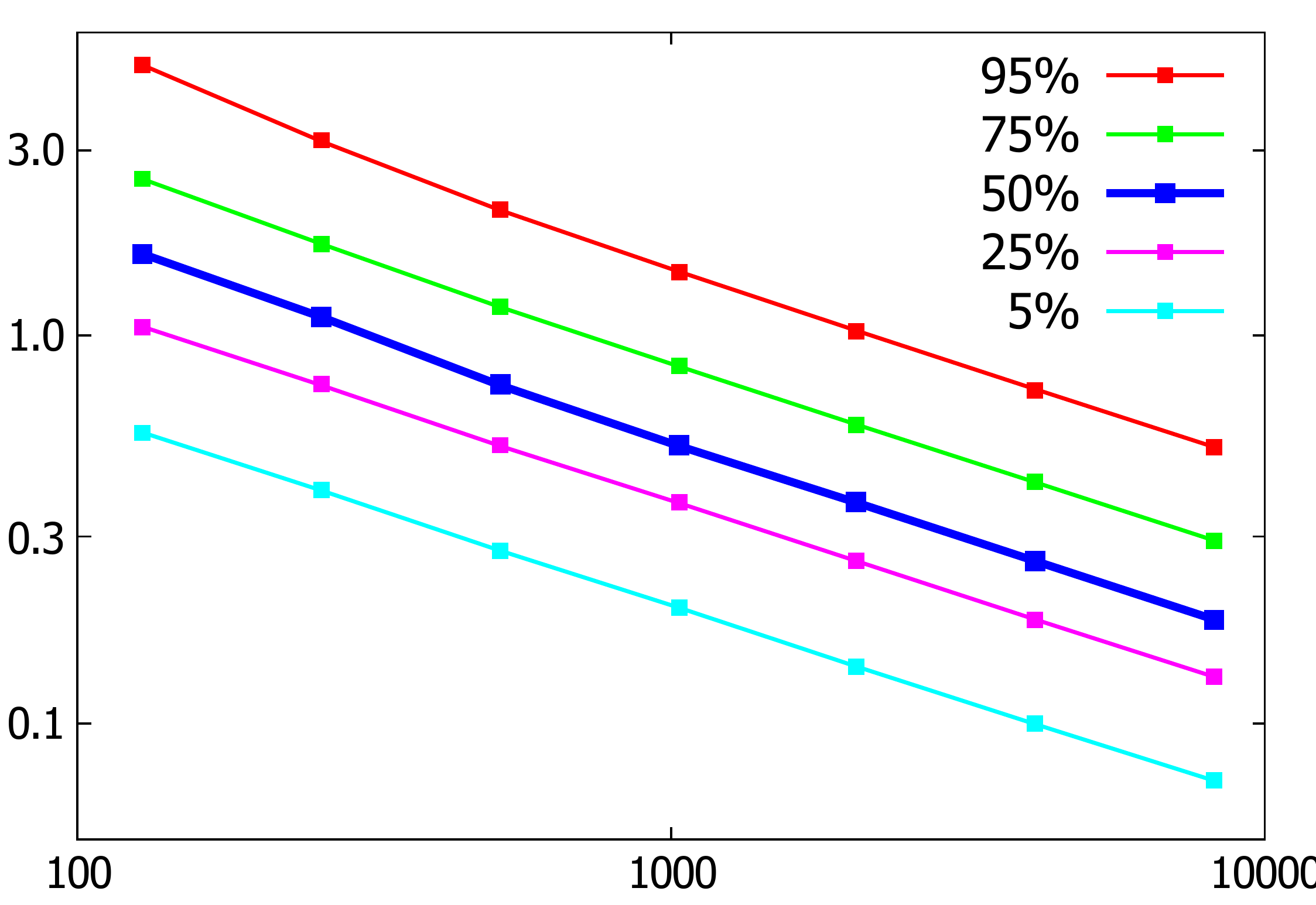}
    \put(-120,-10){sample size $n$}
    \put(-210,45){\rotatebox{90}{$\|\widehat{\mathcal{S}}_4 - \mathcal{S}_4\|_{\mathcal{F}}$}}
    \caption{Error of score function estimator versus sample size for $x \sim \mathcal{N}(0,I_d)$. Top: $\|\widehat{\mathcal{S}}_2 - \mathcal{S}_2\|_{\rm sp}$. Bottom: $\|\widehat{\mathcal{S}}_4 - \mathcal{S}_4\|_{\mathcal{F}}$.}
    \label{fig:gaussian}
\end{figure}

For  the non-Gaussian case, we choose $x \sim 0.5\mathcal{N}({\bf 1}_d,I_d) + 0.5\mathcal{N}(-{\bf 1}_d, I_d)$ where ${\bf 1}_d$ is the all-1 vector and $d=2$.
We also plot the percentiles of the estimation errors in log-log scale in Figure~\ref{fig:mixed}. We can see that LLSFE gives a consistent estimate for the non-Gaussian case too, and the rate is $-0.2587$ for $\|\widehat{\mathcal{S}}_2 - \mathcal{S}_2\|_{\rm sp}$ and $-0.1343$ for $\|\widehat{\mathcal{S}}_4 - \mathcal{S}_4\|_{\mathcal{F}}$, which are also faster than the upper bound in Remark~\ref{rem:remark_1}.

\begin{figure}[h]
    \centering
    \includegraphics[width=0.4\textwidth]{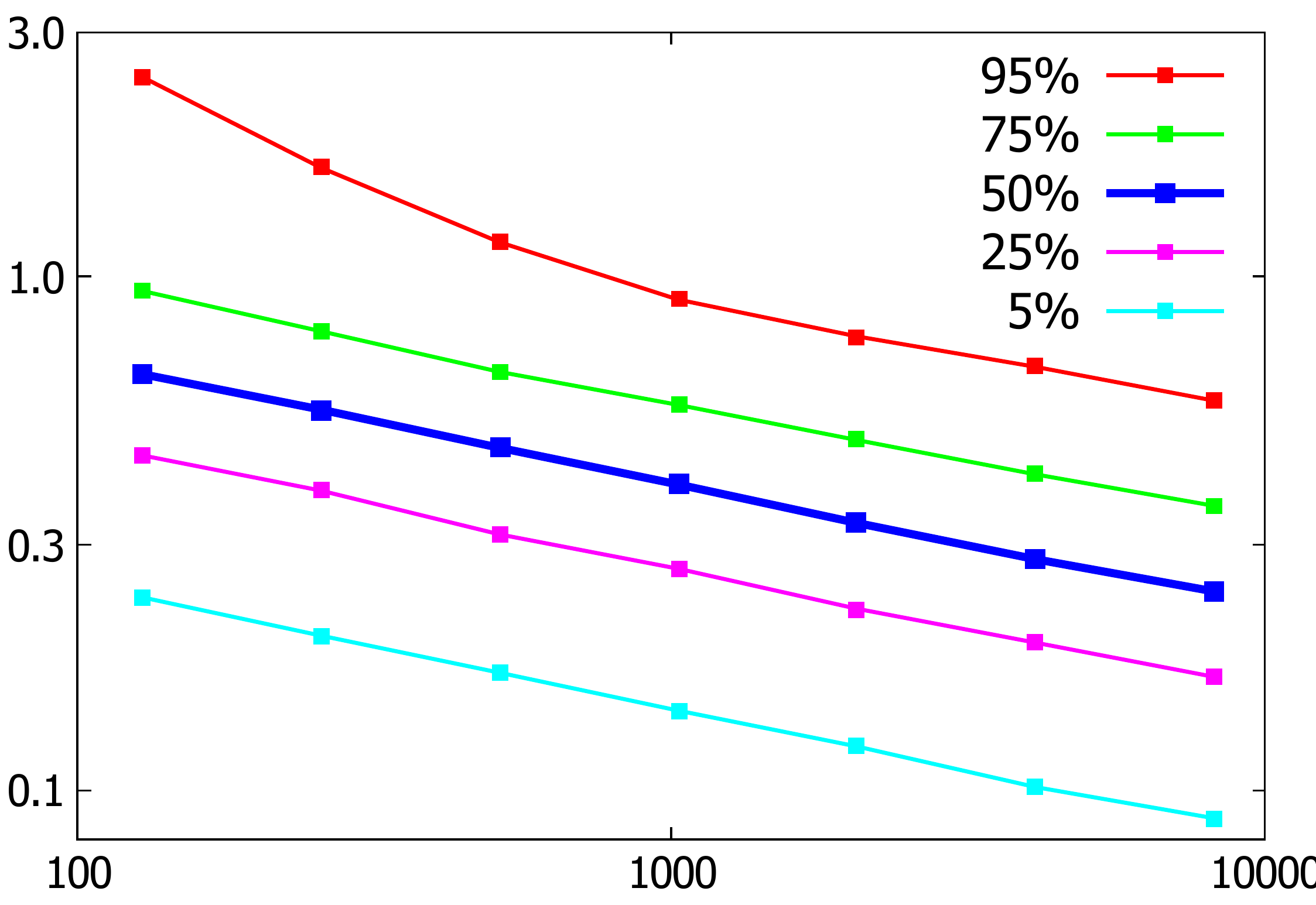}
    \put(-120,-10){sample size $n$}
    \put(-210,45){\rotatebox{90}{$\|\widehat{\mathcal{S}}_2 - \mathcal{S}_2\|_{\rm sp}$}}
    \hspace{0.25in}
    \includegraphics[width=0.4\textwidth]{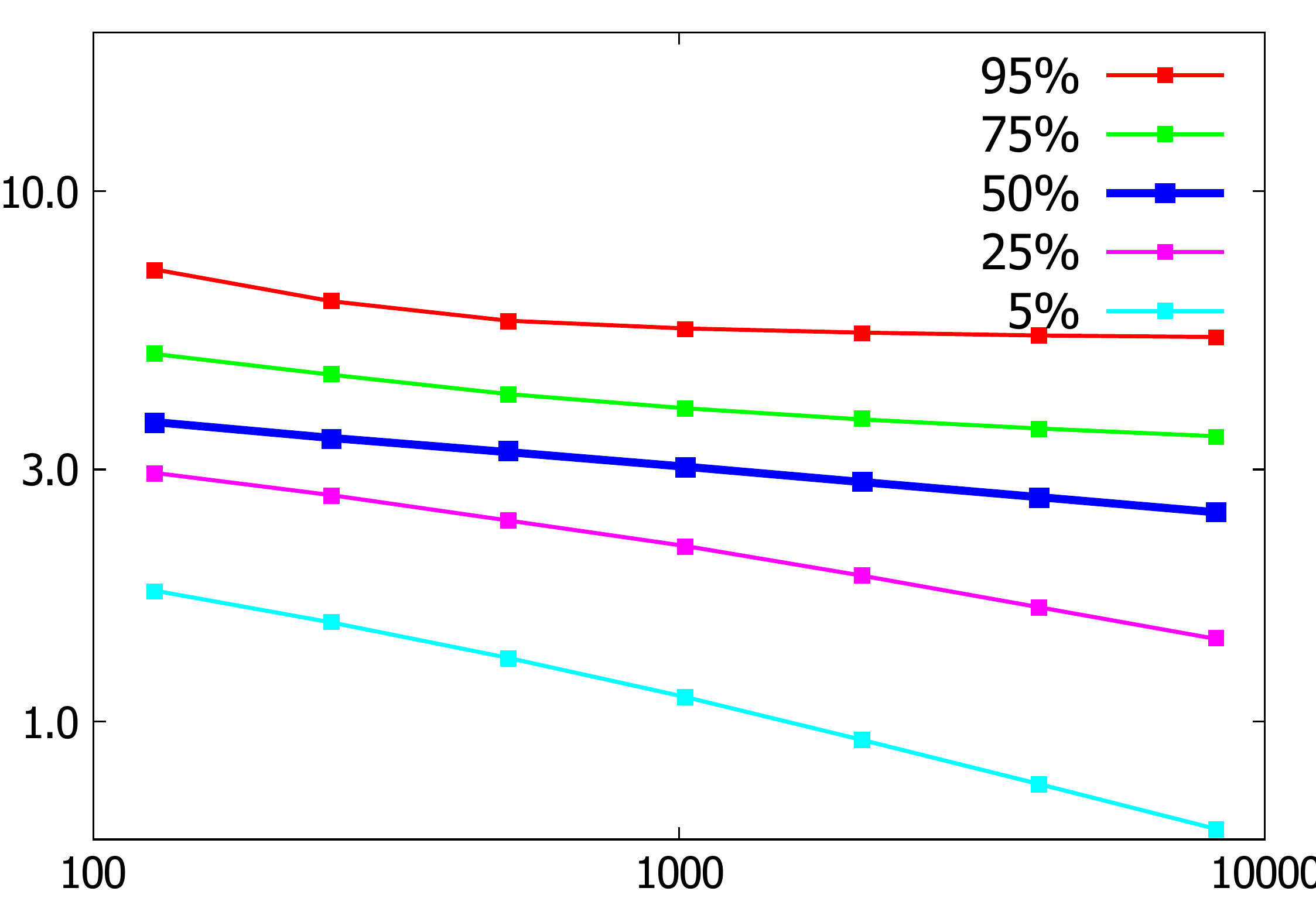}
    \put(-120,-10){sample size $n$}
    \put(-210,45){\rotatebox{90}{$\|\widehat{\mathcal{S}}_4 - \mathcal{S}_4\|_{\mathcal{F}}$}}
    \caption{Error of score function estimator versus sample size for $x \sim 0.5\mathcal{N}({\bf \mu},I_d) + 0.5\mathcal{N}(-{\bf \mu}, I_d)$. Top: $\|\widehat{\mathcal{S}}_2 - \mathcal{S}_2\|_{\rm sp}$. Bottom: $\|\widehat{\mathcal{S}}_4 - \mathcal{S}_4\|_{\mathcal{F}}$.}
    \label{fig:mixed}
\end{figure}

\section{Design of landscape}
\label{sec:landscape}

In this section, we show 
how the proposed density functional estimators can be applied to 
design a loss function with desired properties, 
for  regression problems under a neural network model. 
This gives a novel loss function that does not require the data to be distributed as Gaussian, 
as typically done in existing literature (cf.~Section~\ref{sec:intro} Related work). 

Concretely, we consider the problem of 
training a  
one-hidden-layer neural network where, 
for each input $x \in \reals^d$, 
the corresponding output  is given by
\begin{eqnarray}
	\hat{y}(x) \;\;=\;\; \sum_{i=1}^k w_i \,g(\inner{a_i}{x})\;, 
\end{eqnarray}
with weights are $w_i \in \reals$ and $a_i\in \reals^{d}$, 
non-linear activation is $g:\reals\to\reals$,  
and  the number of hidden neurons is $k \leq d $. 
Given labeled training data $(x,y)$ coming from some distribution, 
a standard approach to training such a network 
is to use the $\ell_2$ loss: 
\begin{eqnarray}
	\ell_2(A) \; =\;  \Expect[\|\widehat{y}(x) - y\|^2]\;, 
\end{eqnarray} 
as the training objective, where
$A$ denotes the weights of the neural network model.
However, traditional optimization techniques on $\ell_2$ can easily get stuck in local optima as empirically shown in~\citep{livni2014computational}. 
This phenomenon can be explained precisely under 
a canonical scenario where the data is generated from a 
 ``teacher neural network'': 
\begin{eqnarray}
	y \; =\; \sum_{i=1}^k w_i^\ast \, g(\inner{a_i^\ast}{x}) \,+\, \eta \;.
	\label{eq:model}
\end{eqnarray}
where the true parameters are $w_i^\ast \in \reals$ and  $a_i^\ast \in \reals^d$, 
and  $\eta$ is a zero-mean noise independent of $x$.
 This assumption that the data also comes from a one-hidden-layer neural network 
is critical in recent mathematical understanding of 
 neural networks, in 
showing the gain of a shallow ResNet by \cite{li2017convergence}, 
various properties of the critical points by \cite{tian2017analytical}, 
and showing that the standard $\ell_2$ minimization is prone to get stuck at 
non-optimal critical points by \cite{ge2017learning}. 
A major limitation of this line of research is that they 
rely critically on the Gaussian assumption on the data $x$. 
The analysis techniques use specific properties of spherical 
Gaussian random variables such that the theoretical findings do not generalize to any other distributions. 
Further, the estimators designed as per those analyses 
fail to give consistent estimates for non-Gaussian data. 

\begin{figure*}[t]
    \centering
    \includegraphics[width=0.42\textwidth]{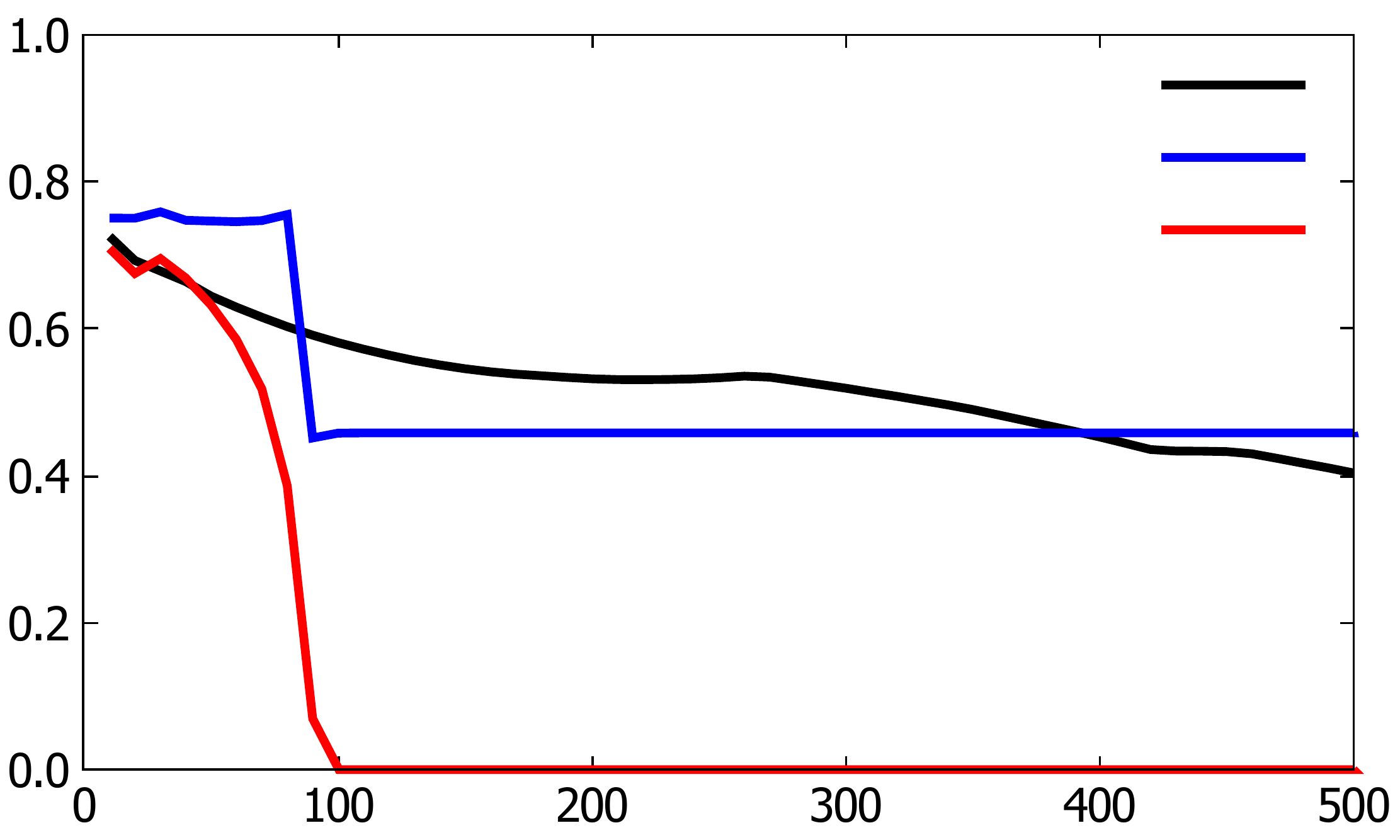}
    \put(-115,-10){Iterations}
    \put(-210,25){\rotatebox{90}{Estimation Error}}
    \put(-97,103){SGD on $\ell_2(\cdot)$}
    \put(-97,93){SGD on $G(\cdot)$}
    \put(-96,83){SGD on $L(\cdot)$}
    \hspace{-0.0in}
    \includegraphics[width=0.42\textwidth]{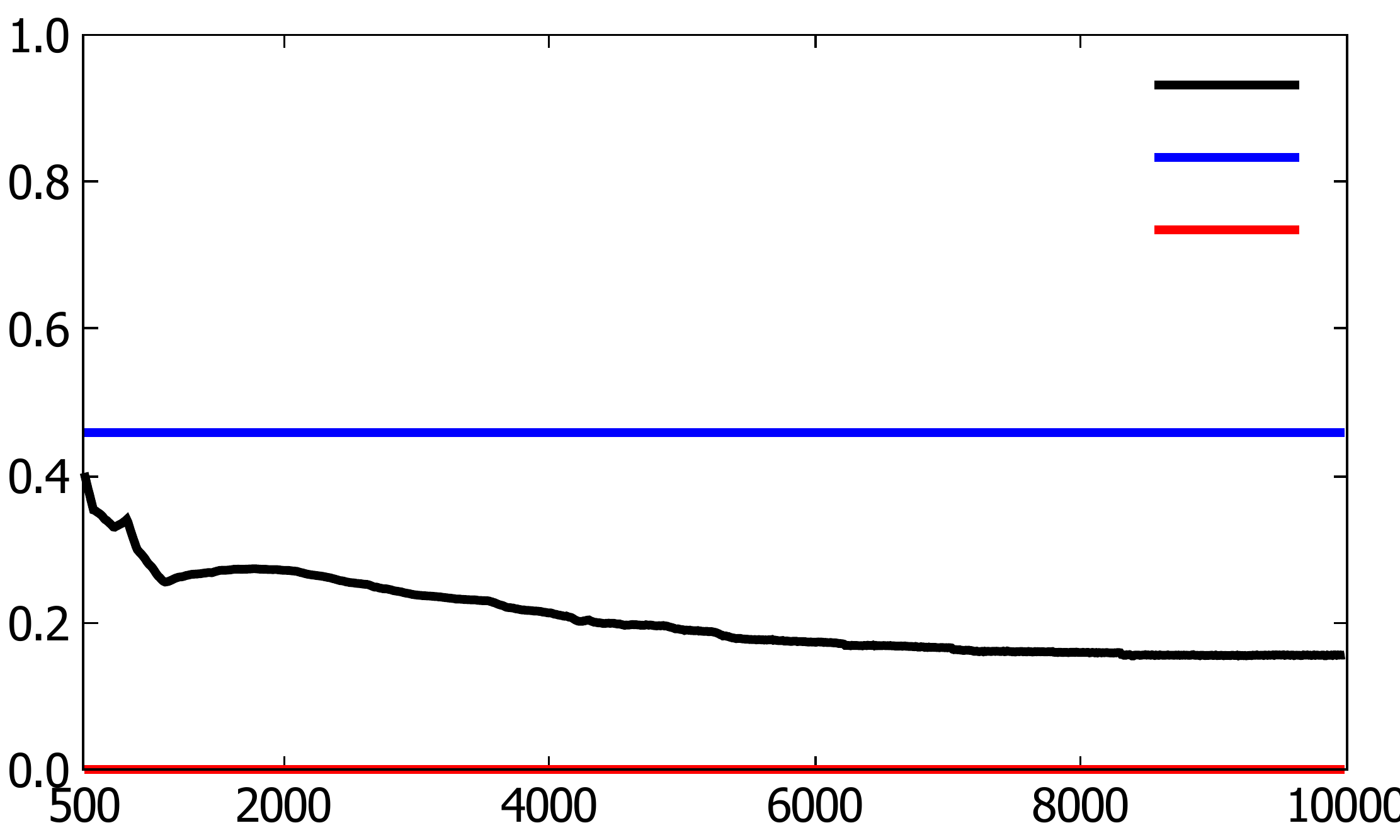}
    \put(-115,-10){Iterations}
    \put(-97,103){SGD on $\ell_2(\cdot)$}
    \put(-97,93){SGD on $G(\cdot)$}
    \put(-96,83){SGD on $L(\cdot)$}
    \caption{
    SGD to learn a one-layer-ReLU network in \prettyref{eq:model} on the proposed objective function $L(A)$ defined in
    \eqref{eq:lossfuncdef} converges to a global minimum with random initialization, whereas on $\ell_2$-loss $\ell_2(A)$ and $G(A)$, it gets stuck at bad local minima.
    Left: First 500 iterations. Right: 500-10,000 iterations.
    }
    \label{fig:exponential}
\end{figure*}
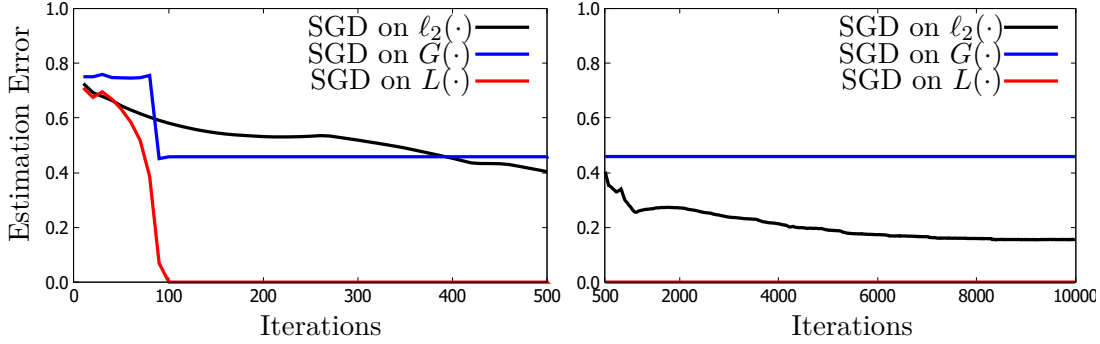

We showcase this limitation in Figure \ref{fig:exponential}, 
where the data is generated from a Laplacian distribution. 
The details of this experiment is provided in Section \ref{subsec:landscape}. 
Minimizing $\ell_2$ loss converges slowly and gets stuck at 
sub-optimal critical points, consistent with previous observations \citep{li2017convergence}. 
To overcome this weakness 
\cite{ge2017learning} proposed  
applying Stochastic Gradient Descent (SGD) on a novel loss function $G(A)$
designed from the analysis under the Gaussian assumption. 
This fails to converge to an optimal critical point for non-Gaussian distributions. 
To overcome this limitation, we propose a novel loss function $L(A)$ that generalizes to 
a broad class of distributions. 

We focus on the task of recovering the weights $a_i^*$'s, and 
denote the set by a matrix $A^\top =[a_1|\ldots|a_k] \in \reals^{d \times k} $. 
The scalar weights $w_i^*$'s can be separately estimated using standard least squares, once $A$ has been recovered. 
We propose applying SGD on a new loss function $L(A)$, defined as
\begin{align}
	& L(A) \;= \sum_{i,j \in [k], i \neq j} \Expect[y \cdot t_1(x,a_i,a_j)] - \mu \sum_{i \in [k]} \Expect[y \cdot t_2(x,a_i)] + \lambda \sum_{i \in [k]} (\norm{a_i}-1)^2 \;,
	\label{eq:lossfuncdef}
\end{align}
where $\mu,\lambda>0$ are regularization coefficients, and 
\begin{eqnarray}
t_1(x,u,v) &=& \calS_4(x)(u,u,v,v), \,\notag\\
t_2(x,u) &=& \calS_4(x)(u,u,u,u), \; u,v \in \reals^d.
\end{eqnarray}
are the applications of the score functions 
 $\mathcal{S}_m(x) = \nabla^{(m)} f(x) / f(x)$ on the weight vectors $a_i$'s that we are optimizing over, 
 i.e.~$\calS_4(x)(u,v,w,z) = (1/f(x)) \sum_{i_1,i_2,i_3,i_4} \nabla_{x_{i_1}x_{i_2}x_{i_3}x_{i_4}}f(x)  u_{i_1} v_{i_2} w_{i_3} z_{i_4} $.  
 We provide  formulas for some simple distributions below. 
 
 \begin{example}[Gaussian]
\label{examp:gaussian}
If $x \sim \calN(0,I_d)$, we have that $t_2^{(G)}(x,u)=(u^\top x)^4 - 6 \norm{u}^2 (u^\top x)^2 + 3\norm{u}^4$ and $t_1^{(G)}(x,u,v)=(u^\top x)^2(v^\top x)^2 - \norm{u}^2(v^\top x)^2 - 4(u^\top x)(v^\top x)(u^\top v) - \norm{v}^2 (u^\top x)^2+\norm{u}^2 \norm{v}^2 + 2 (u^\top v)^2$.
\end{example}

\begin{example}[Mixture of Gaussians]
\label{examp:mgg}
If $x \sim p \calN(\mu_1,I_d)+(1-p)\calN(\mu_2,I_d)$, we have that $t_1(x,u,v)=p_1 t_1^{(G)}(x-\mu_1,u,v)+(1-p_1)t_1^{(G)}(x-\mu_2,u,v)$ where the posterior $p_1 \define \frac{p\calN(\mu,I_d)}{ p \calN(\mu_1,I_d)+(1-p)\calN(\mu_2,I_d)}$. Similarly for $t_2$.
\end{example}

The proposed  $L(\cdot)$ is carefully designed to ensure that the loss surface has a desired landscape with no local minima. 
Here, we give the intuition behind the design principle, and 
make it precise in the main results of Theorems \ref{thm:landscape} and \ref{thm:finitesample}. 
This landscape explains the experimental superiority of $L(\cdot)$ 
in Figures \ref{fig:exponential} and \ref{fig:training}.
Suppose $k=d$ and $a_i^*$'s are orthogonal vectors. 
After some calculus, 
an alternative characterization for $L$ is given by
\begin{eqnarray}
L(A) &=& \sum_{i \in [d]} w_i^\ast \Expect[g^{(4)}(\inner{a_i^\ast}{x})] \sum_{j,k \in [d], j\neq k} \inner{a_i^\ast}{a_j}^2 \inner{a_i^\ast}{a_k}^2 \,\notag\\
&&\; -\mu \sum_{i,j \in [d]} w_i^\ast \Expect[g^{(4)}(\inner{a_i^\ast}{x})]\inner{a_i^\ast}{a_j}^4  +\lambda \sum_{i \in [d]}(\norm{a_i}-1)^2  \nonumber \\
&=&\sum_{i \in [d]} \kappa_i^\ast\sum_{j,k \in [d], j\neq k} \inner{a_i^\ast}{a_j}^2 \inner{a_i^\ast}{a_k}^2  -\mu \sum_{i,j \in [d]} \kappa_i^\ast \inner{a_i^\ast}{a_j}^4+\lambda \sum_{i \in [d]}(\norm{a_i}^2-1)^2\;, \label{eq:temp}
\end{eqnarray}
for  scalar $\kappa_i^*=w_i^\ast \Expect[g^{(4)}(\inner{a_i^\ast}{x})] $ that does not depend on the variables we optimize over. 

Notice that when the weights are recovered up to a permutation, that is 
$a_i=\pm a^\ast_{\pi(i)}$ for some permutation $\pi$, the first term in \prettyref{eq:temp} equals zero. 
We can show that these are the only possible local minima in the minimization of the first term under unit-norm constraints, whenever all $\kappa_i^\ast=1$.
 Thus in order to account for this weighted tensor based loss and to avoid spurious local minima, the regularization term $\mu \sum_{i,j \in [d]} \kappa_i^\ast \inner{a_i^\ast}{a_j}^4$ forces these
 spurious minima to lie close to a permutation of $a_i^\ast$ up to a sign flip. 
 This is made precise in the characterization of the landscape of $L(\cdot)$ 
 in the proof of Theorem \ref{thm:landscape}.
The proof strategy is inspired by the landscape analysis technique of  \cite{ge2017learning}, 
where a similar analysis was done for Gaussian data $x$. 


\subsection{Theoretical results}

We now formally state the assumptions for our theoretical results.
\begin{assumption}
\label{assump:landscapeassump}
\begin{enumerate}
        \item[(a)] The ground-truth parameters $w_i^\ast,a_i^\ast $ are such that $w_i^\ast \Expect[g^{(4)}(\inner{a_i^\ast}{x})]$ has the same sign for all $i \in [k]$.
        \item[(b)] Defining $\kappa_i^\ast=w_i^\ast \Expect[g^{(4)}(\inner{a_i^\ast}{x})]$ and $\kappa^\ast=\max_i \kappa_i^\ast/(\min_i \kappa_i^\ast)$, 
        we choose $\mu<c/\kappa^\ast$ and $\lambda \geq \kappa^\ast_{\max}/c$ for $c \leq 0.01$.
        \item[(c)] $k=d$ and $A \in \reals^{d \times d}$ is an orthogonal matrix.
    \end{enumerate}
\end{assumption}

The following theorem characterizes the landscape of $L(\cdot)$.
\begin{theorem}
	\label{thm:landscape}
	Under \prettyref{assump:landscapeassump}, the objective function $L(\cdot)$ satisfies that
	\begin{enumerate}
	\item All local minima of $L$ are also global. Furthermore, all approximate local minima are also close to the global minimum. More concretely, for $\varepsilon>0$, let $A$ satisfy that
	 \begin{align*}
	 \norm{\nabla L(A)} \leq \varepsilon \text{ and } \lambda_{\min}\pth{\nabla^2 L(A)} \geq - \tau,
	 \end{align*}
	where $\tau=c \min \sth{\mu \kappa_{\min}^\ast/(\kappa^\ast d),\lambda}$. Then $A=P D A^\ast+ EA^\ast$, where $P$ is a permutation matrix, $D$ is a diagonal matrix with 
	$D_{ii} \in  \{ \pm 1 \pm O(\mu \kappa_{\max}^\ast/\lambda)\}$, and $|E|_\infty \leq O\pth{\varepsilon/(\kappa_{\min}^\ast)}$.

	 \item Any saddle point $A$ has a strictly negative curvature, \ie $\lambda_{\min}\pth{\nabla^2 L(A)} \leq -\tau$.

	\end{enumerate}
\end{theorem}
\begin{remark}
For the case when $a_1,\ldots,a_k$ are linearly independent with $k<d$, similar conclusion hold (see \prettyref{app:generalk}).
\end{remark}


\subsection{Finite Sample Regime}
In the finite sample regime, we replace the population expectation in \prettyref{eq:lossfuncdef} with empirical expectation $\hat{\Expect}$ and optimize on the corresponding loss $\hat{L}$. The following theorem establishes that $\hat{L}$ also exhibits similar landscape properties as that of $L$ (under some mild technical assumptions outlined in \prettyref{assump:finiteassump} in Appendix~\ref{app:sec2proof}).

\begin{theorem}
\label{thm:finitesample}
Assume that \prettyref{assump:landscapeassump} and 
\prettyref{assump:finiteassump} (defined in Appendix \ref{sec:finiteproof}) hold. 
Then there exists a polynomial $\mathrm{poly}(d,1/\varepsilon)$ 
such that whenever $n \geq \mathrm{poly}(d,1/\varepsilon)$, with high probability, $\hat{L}$ exhibits the same landscape properties as that of $L$, established in \prettyref{thm:landscape}.
\end{theorem}

A major bottleneck in applying the proposed loss \prettyref{eq:lossfuncdef} directly to real data is that the knowledge of 
the probability density  function of the data $x$ is required. 
As we saw in the Examples \ref{examp:gaussian} and \ref{examp:mgg}, 
the loss function $t_1$ and $t_2$ depends on the pdf of $x$. 
In the next section, we show how we can combine the 
LLSFE  to compute (the gradients of) those functions to introduce a novel consistent estimator with a desirable landscape. 

\section{Experiments}
\label{sec:experiment}

\subsection{Landscape of $L(\cdot)$}
\label{subsec:landscape}

In this  simulation, we show that the landscape of the loss function $L(A)$ is well-behaved, if we know the score function $\mathcal{S}_4(x)$. 
We choose $x = (x_1, \dots, x_d)$, where $x_i$ are i.i.d. symmetric exponential distributed random variables, i.e., $f(x_i) = (1/2)\exp\{-|x_i|\}$. The fourth-order score function is given by $\mathcal{S}_4(x) = {\rm sgn}(x)^{\otimes 4}$. We compare our loss function $L(A)$ with an $\ell_2$-loss, $\ell(\cdot)$, as well as the loss function $G(\cdot)$ proposed in~\citep{ge2017learning}, and evaluate the performance through the parameter error (which verifies if $A^{* -1}A$ is close to a permutation matrix)
\begin{eqnarray}
e(A) &=& \min\{1-\min_i \max_j |(A^{* -1}A)_{ij}|,  1-\min_j \max_i |(A^{* -1}A)_{ij}|\}.
\end{eqnarray}
For the experiment, we choose $A^* = I_d$, $w^* = 1$, $\sigma = {\rm ReLU}$, $k = d = 50$ and use full-batch gradient descent with sample size 8192 and learning rate $\eta = 5\times 10^{-3}$ for $\ell_2$ loss and $\eta = 5 \times 10^{-5}$ for $L(A)$ and $G(A)$. Regularization parameter is $\mu = 30$ for both $L(A)$ and $G(A)$.
The results are illustrated in Figure~\ref{fig:exponential}, which shows that 
$(i)$ $\ell_2(\cdot)$ converges slowly and to a suboptimal critical point indicating the existence of local minima; 
$(ii)$ $G(\cdot)$ converges to a suboptimal critical point due to the mismatched Gaussian assumption; and 
$(iii)$ $L(\cdot)$ converges to a global minima.

\begin{figure}[h]
    \centering
    \includegraphics[width=0.42\textwidth]{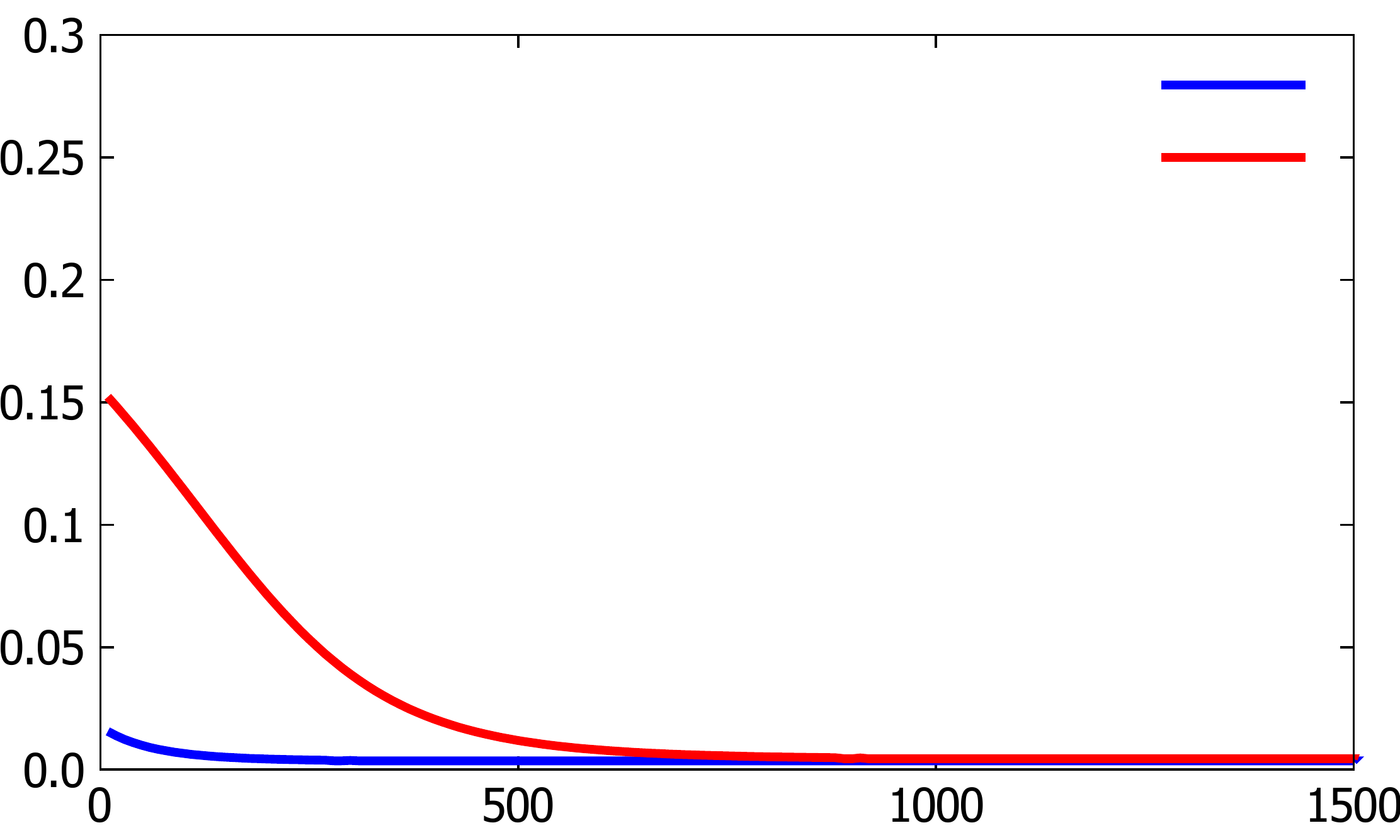}
    \put(-105,-10){Iterations}
    \put(-210,25){\rotatebox{90}{Parameter Error}}
    \put(-55,92){$\widehat{L}(\cdot)$}
    \put(-56,102){$G(\cdot)$}
    \hspace{0.0in}
    \includegraphics[width=0.42\textwidth]{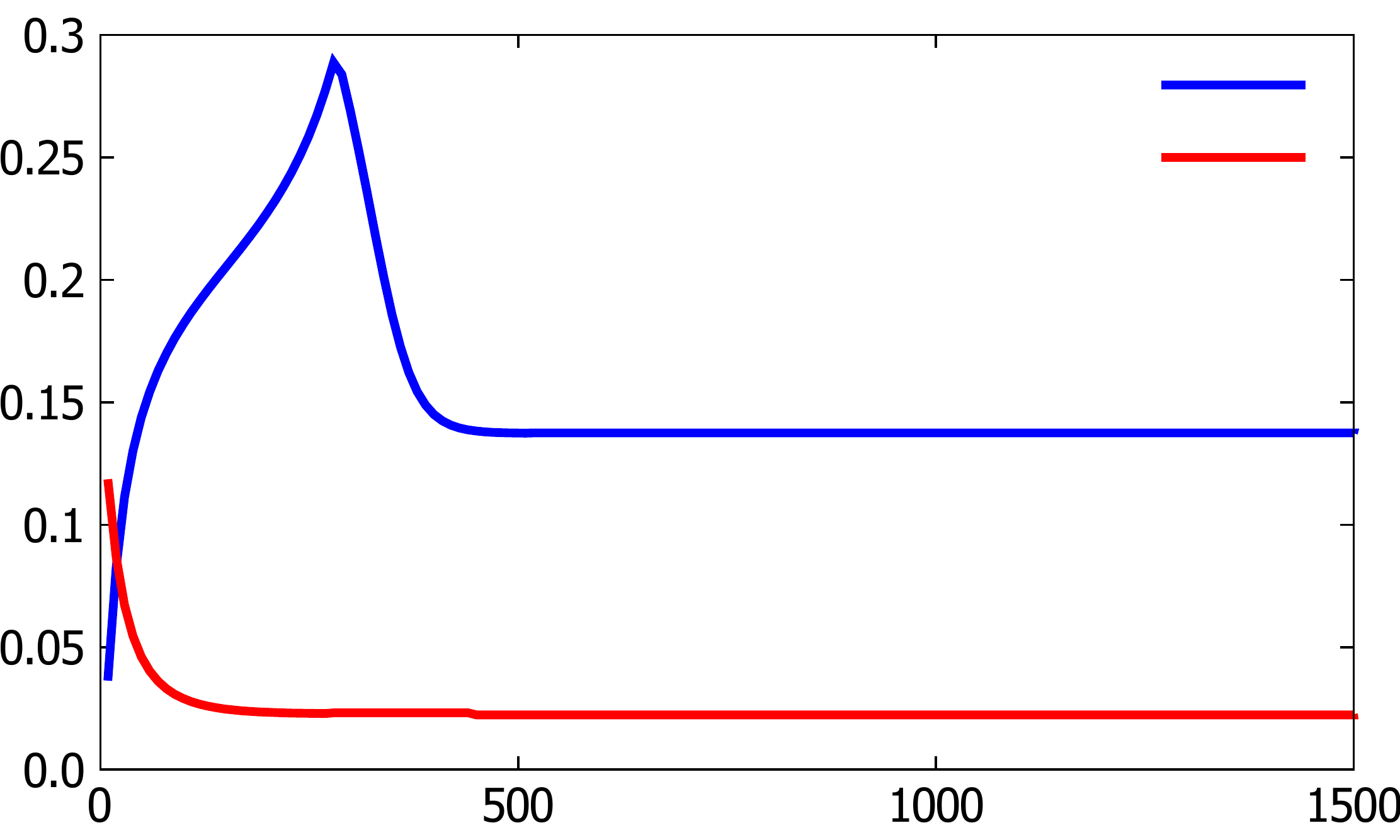}
    \put(-110,-10){Iterations}
    \put(-55,92){$\widehat{L}(\cdot)$}
    \put(-56,102){$G(\cdot)$}
    \caption{Learning curve of objective function $G(A)$ (blue line) and LLSFE based objective function $L(A)$~\eqref{eq:lossfuncdef} (red line). Top: $x$ is Gaussian. Bottom: $x$ is Gaussian-mixture.}
    \label{fig:training}
\end{figure}
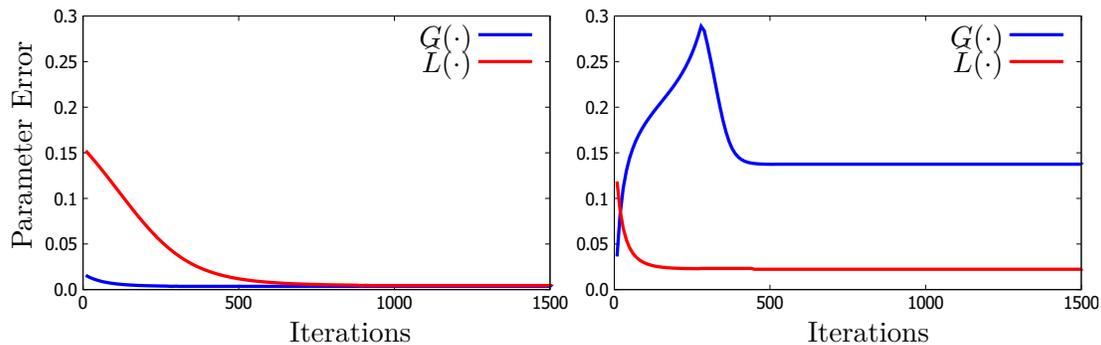

\subsection{Combine with LLSFE}
\label{subsec:training}

Now we use our estimator LLSFE to construct the empirical loss $\widehat{L}(A)$ to train a one-hidden-layer neural network~\eqref{eq:model}.
The setting of this experiment is same as that of Subsection~\ref{subsec:landscape} with $k=d=2$ for simplicity.

In the left panel of Figure.~\ref{fig:training}, we choose Gaussian input $x \sim \mathcal{N}(0, I_d)$ so that the loss $G(A)$ coincides with $L(A)$ if the ground truth $\mathcal{S}_4(x)$ is known. We can see that using estimation error using $L(\cdot)$ operates close to that of the ground-truth $G(\cdot)$
In the right panel of Figure~\ref{fig:training}, we choose $x \sim  0.5 \mathcal{N}({\bf 1}_d, I_d) + 0.5 \mathcal{N}(-{\bf 1}_d, I_d)$. In this case, $G(A)$ converges to a local minimum, thus incurring higher parameter error, whereas LLSFE-based objective function converges to the global minima very quickly. 
This confirms that when the data is not coming from a Gaussian distribution, it is critical to use properly matched estimator, 
which is provided by the proposed LLSFE approach. 

\section{Conclusion}
\label{sec:conclusion}

Stochastic gradient descent is the dominant method for training neural networks. 
As SGD on the standard $\ell_2$ loss fails to converge to the true parameters of the ``teacher'' networks, from which the data is generated, 
there have been significant efforts to  design 
a loss function with a good landscape. 
However, those new loss functions are typically tailored only for Gaussian distributions; 
a common assumption in theory of neural networks, but far from the real data.   

To bridge this gap, we propose a new framework for designing the landscape for general smooth distributions.
Using local likelihood density estimators, which can capture the local geometry of the probability density function, 
we introduce a novel estimator for {\em score functions} which $(i)$ involve higher-order derivatives of the input pdf and $(ii)$ are critical in the landscape design. 
This resolves one of the challenges in generalizing 
the Gaussian assumption, namely score function estimation. 

There are other challenges in removing the Gaussian assumption in the analysis of more complicated networks, for example in \cite{brutzkus2017globally,li2017convergence}. 
Innovative analysis techniques are needed to complete the generalization of the Gaussian assumption, which is a promising direction for future research. Also, the time complexity of our approach is polynomial of the dimension of input, but the exact order of is unknown. Further improvement on time complexity could be a promising future research direction.

\section*{Acknowledgement}
This work is partially supported by NSF grant 1815535, NSF CNS-1718270 and the Army Research Office under grant W911NF1810332,

\bibliographystyle{apalike}
\bibliography{reference}

\balance
\clearpage

\appendix

\section{Proof of Section~\ref{sec:estimation}}

\subsection{Proof of Theorem~\ref{thm:theorem_1}}

\begin{proof}
We rewrite the spectral norm error in terms of the polynomial representations~\eqref{eq:S_poly} and~\eqref{eq:hat_S_poly} as
\begin{eqnarray}
&&\| \widehat{\mathcal{S}_m}(x) - \mathcal{S}_m(x) \|_{\rm sp} \,\notag\\
&\leq& \sum_{\lambda \in \Lambda_m} c_m(\lambda) \, \| {\rm sym}(\bigotimes_{j \in \lambda} \widehat{\mathcal{A}_j^{(p)}}) - {\rm sym}(\bigotimes_{j \in \lambda} \mathcal{G}_j)\|_{\rm sp} \,\notag\\
&\leq& \sum_{\lambda \in \Lambda_m} c_m(\lambda) \, \| \bigotimes_{j \in \lambda} \widehat{\mathcal{A}_j^{(p)}} - \bigotimes_{j \in \lambda} \mathcal{G}_j\|_{\rm sp} \;,
\label{eq:eq_1}
\end{eqnarray}
where the last inequality comes from the fact that $\|{\rm sym}(\mathcal{T})\|_{\rm sp} \leq \|\mathcal{T}\|_{\rm sp}$. Then we study each term in~\eqref{eq:eq_1}. For simplicity of notation, denote the estimation error $\mathcal{E}_j^{(p)} \triangleq \widehat{\mathcal{A}_j^{(p)}} - \mathcal{G}_j$, then we have
\begin{eqnarray}
&&\| \bigotimes_{j \in \lambda} \widehat{\mathcal{A}_j^{(p)}} - \bigotimes_{j \in \lambda} \mathcal{G}_j\|_{\rm sp} \,\notag\\
&=& \| \bigotimes_{j \in \lambda} ( \mathcal{E}_j^{(p)} + \mathcal{G}_j ) - \bigotimes_{j \in \lambda} \mathcal{G}_j\|_{\rm sp} \,\notag\\
&=& \| \sum_{\nu \subset \lambda} \left( \left( \bigotimes_{j \in \nu}(\mathcal{E}_j^{(p)})\right) \otimes \left( \bigotimes_{j \in \lambda \setminus \nu} \mathcal{G}_j\right) \right) -  \bigotimes_{j \in \lambda} \mathcal{G}_j \|_{\rm sp} \,\notag\\
&=& \| \sum_{\nu \subset \lambda, \nu \neq \emptyset} \left( \left( \bigotimes_{j \in \nu}(\mathcal{E}_j^{(p)})\right) \otimes \left( \bigotimes_{j \in \lambda \setminus \nu} \mathcal{G}_j\right) \right) \|_{\rm sp} \,\notag\\
&\leq& \sum_{\nu \subset \lambda, \nu \neq \emptyset} \|\left( \bigotimes_{j \in \nu}(\mathcal{E}_j^{(p)})\right) \otimes \left( \bigotimes_{j \in \lambda \setminus \nu} \mathcal{G}_j\right) \|_{\rm sp} \,\notag\\
&\leq& \sum_{\nu \subset \lambda, \nu \neq \emptyset} \left(\, ( \prod_{j \in \nu} \|\mathcal{E}_j^{(p)} \|_{\rm sp}) \times ( \prod_{j \in \lambda \setminus \nu} \| \mathcal{G}_j\|_{\rm sp} )\,\right) \;.\notag\\
\label{eq:eq_2}
\end{eqnarray}

%

Now we study the spectral norm of $\mathcal{E}_j^{(p)}$, which can be upper bounded by the Frobenius norm. Then by Lemma~\ref{lem:lemma_1}, we have,
\begin{eqnarray}
&& \|\mathcal{E}_j^{(p)}\|_{\rm sp} \leq \|\mathcal{E}_j^{(p)}\|_{\mathcal{F}} = \sqrt{\sum_{i_1, \dots, i_j} \left( \mathcal{E}_j^{(p)} \right)_{(i_1, \dots, i_j)}^2} \,\notag\\
&=& O(d^{j/2}h^{p+1-j}) + O_p(d^{j/2}(nh^{d+2j})^{-1/2}).
\end{eqnarray}

Since for any $j \leq m$, we have $h^{p+1-j} \to 0$ and $nh^{d+2j} \to \infty$ as $n \to \infty$. So for sufficiently large $n$, we have $\sum_{j \in \lambda} \|\mathcal{E}_j^{(p)}\|_{\rm sp} \leq 1$ with high probability. Then, plug it into~\eqref{eq:eq_2}, we get
\begin{eqnarray}
&&\| \bigotimes_{j \in \lambda} \widehat{\mathcal{A}_j^{(p)}} - \bigotimes_{j \in \lambda} \mathcal{G}_j\|_{\rm sp} \,\notag\\
&\leq& \sum_{\nu \subset \lambda, \nu \neq \emptyset} \left(\, ( \prod_{j \in \nu} \|\mathcal{E}_j^{(p)} \|_{\rm sp}) \times \prod_{j \in \lambda \setminus \nu} C_j \,\right) \,\notag\\
&\leq& C \sum_{\nu \subset \lambda, \nu \neq \emptyset}  \prod_{j \in \nu} \|\mathcal{E}_j^{(p)} \|_{\rm sp} \,\notag\\
&=& C \left( \prod_{j \in \lambda} (1+\|\mathcal{E}_j^{(p)} \|_{\rm sp}) - 1 \right) \,\notag\\
&\leq& C \left( \exp\{\sum_{j \in \lambda} \|\mathcal{E}_j^{(p)}\|_{\rm sp}\} - 1 \right) \,\notag\\
&\leq& 2 C \sum_{j \in \lambda} \|\mathcal{E}_j^{(p)}\|_{\rm sp} \,\notag\\
&=& O(d^{j_{\rm max}/2}h^{p+1-j_{\rm max}}) + O_p(d^{j_{\rm max}/2}(nh^{d+2j_{\rm max}})^{-1/2})
\end{eqnarray}
here constant $C = \max_{\nu} \prod_{j \in \lambda \setminus \nu} C_j$ and $j_{\rm max} = \max\{j: j \in \lambda\}$. The last inequality comes from the fact that $e^y - 1\leq 2y$ for any $y \leq 1$. Since $\lambda$ is a partition of integer $m$, we have $j_{\rm max} \leq m$, and the equation holds if and only if $\lambda = \{m\}$. Therefore the only term in~\eqref{eq:eq_1} that achieves $O(d^{m/2}h^{p+1-m}) + O_p(d^{m/2}(nh^{d+2m})^{-1/2})$ is $\|\widehat{\mathcal{A}}_m^{(p)} - \mathcal{G}_m\|_{\rm sp}$, with $c_m(\lambda) = 1$. Therefore, we complete the proof.
\end{proof}

\section{Proofs of \prettyref{sec:landscape}}
\label{app:sec2proof}

The key technical lemma behind our results is the Stein's lemma and its generalizations which we present below.
\begin{lemma}[\cite{stein1972bound}]
Let $x \sim \calN(0,I_d)$ and $g:\reals^d \to \reals$ be such that both $\Expect[\nabla g(x)]$ and $\Expect[g(x)x]$ exist and are finite. Then
\begin{eqnarray}
\Expect[g(x)x]=\Expect[\nabla_x g(x)].
\end{eqnarray}
\end{lemma}
The following lemma generalizes Stein's lemma to more general distributions and higher-order derivatives.
\begin{lemma}[\cite{sedghi2014provable}]
\label{lmm:generalstein}
Let $m \geq 1$ and $\calS_m(x)$ be defined as in \prettyref{eq:defscore}. Then for any $g:\reals^d \to \reals$ satisfying some regularity conditions, we have
\begin{eqnarray}
\Expect[g(x)\cdot \calS_m(x)]=\Expect[\nabla_x^{(m)}g(x)].
\end{eqnarray}
\end{lemma}
The following theorem gives an alternate characterization of the loss function $L$ and is the key step in the proof of \prettyref{thm:landscape}.
\begin{theorem}
\label{thm:alternate}
The loss function $L(\cdot)$ defined in \prettyref{eq:lossfuncdef} satisfies that
\begin{eqnarray}
L(A) &=& \sum_{i \in [d]} w_i^\ast \Expect[g^{(4)}(\inner{a_i^\ast}{x})] \sum_{j,k \in [d], j\neq k} \inner{a_i^\ast}{a_j}^2 \inner{a_i^\ast}{a_k}^2 \,\notag\\
&&\;-\mu \sum_{i,j \in [d]} w_i^\ast \Expect[g^{(4)}(\inner{a_i^\ast}{x})]\inner{a_i^\ast}{a_j}^4 +\lambda \sum_{i \in [d]}(\norm{a_i}-1)^2 \;.
\end{eqnarray}
\end{theorem}
\begin{proof}
Since $\eta$ is zero-mean and independent of $x$, we have that
\begin{align}
\Expect[y \cdot \calS_m(x)]= \sum_{i \in [k]} w_i^\ast \Expect[g(\inner{a_i^\ast}{x}\cdot \calS_m(x))] \;,
\label{eq:hmm}
\end{align}
\end{proof}
Putting $m=4$ in \prettyref{lmm:generalstein}, in view of \prettyref{eq:hmm}, we obtain that
\begin{eqnarray}
\Expect[y \cdot \calS_4(x)]=  \sum_{i \in [k]} w_i^\ast \Expect[g^{(4)}(\inner{a_i^\ast}{x})] (a_i^\ast)^{\otimes 4}.
\end{eqnarray}
Thus for any fixed $a_j,a_k$, we have
\begin{eqnarray}
&&\Expect[y \cdot \calS_4(x)(a_j,a_j,a_k,a_k)]=\Expect[y \cdot t_1(x)] \,\notag\\
&=&\sum_{i \in [k]} w_i^\ast \Expect[g^{(4)}(\inner{a_i^\ast}{x})] \inner{a_i^\ast}{a_j}^2 \inner{a_i^\ast}{a_k}^2,\\
&&\Expect[y \cdot \calS_4(x)(a_j,a_j,a_j,a_j)]=\Expect[y \cdot t_2(x)] \,\notag\\
&=&\sum_{i \in [k]} w_i^\ast \Expect[g^{(4)}(\inner{a_i^\ast}{x})] \inner{a_i^\ast}{a_j}^4.
\end{eqnarray}
Now summing over $j,k$ finishes the proof.

\subsection{Proof of \prettyref{thm:landscape}}
\begin{proof}
The proof directly follows from Theorem 2.3 of \cite{ge2017learning} and \prettyref{thm:alternate}.
\end{proof}

\subsection{Proof of \prettyref{thm:finitesample}}
\label{sec:finiteproof}

We formally state our assumptions for the finite sample landscape analysis below.

\begin{assumption}
\label{assump:finiteassump}
\begin{enumerate}
\item[(a)] $\norm{x}$ has exponentially decaying tails, \ie
\begin{eqnarray}
\prob{\norm{x}^2 \geq t} \leq K_1 e^{-K_2 t^2}, \quad \forall t \geq 0,
\end{eqnarray}
for some constants $K_1,K_2 >0$.
\item[(b)] Let $l(x,y,A)$ be such that $L(A)=\Expect[l(x,y,A)]+\lambda \sum_{i \in [k]}(\norm{a_i}^2-1)^2$. Then there exists a constant $K>0$ which is at most a polynomial in $d$ and a constant $p \in \naturals$ such that
\begin{eqnarray}
\norm{\nabla_A l(x,y,A)} &\leq& K \norm{x}^p, \notag\\
 \norm{\nabla^2_A l(x,y,A)} &\leq& K \norm{x}^p,
\end{eqnarray}
for all $A$ such that $\norm{A_i} \leq 2$.
\end{enumerate}
\end{assumption}

In order to establish that the gradient and the Hessian of $L$ are close to their finite sample counterparts, we first consider its truncated version $L_T$ defined as
\begin{align}
L_T \define \Expect[l(x,y,A) \mathbbm{1}_E], \quad E \define \{\norm{x} \leq R  \},
\label{eq:trunc}
\end{align}
where $R=Cd \log(1/\varepsilon)$ for some $\varepsilon <0$. It follows that $L_T$ is well behaved and exhibits uniform convergence of empirical gradients/Hessians to its population version~\cite{ge2017learning} for $A$ with bounded norm. Then \prettyref{thm:finitesample} follows from showing that the gradient and the Hessian of $L_T$ are close to that of $L$ as well in this setting, which we prove in \prettyref{lmm:truncclose}. Next we combine this result with Lemma E.5 of \cite{ge2017learning} which shows that $A$ with large row norms must also have large gradients and hence cannot be local minima. First we define $L_T$
\begin{lemma}
\label{lmm:truncclose}
Let $L_T$ be defined as in \prettyref{eq:trunc} and \prettyref{assump:finiteassump} hold. Then for a sufficiently large constant $C$ and a sufficiently small $\varepsilon>0$, we have that
\begin{eqnarray}
\norm{\nabla L(A)- \nabla L_T(A)}_2 \leq \varepsilon,\\
\norm{\nabla^2 L(A)- \nabla^2 L_T(A)}_2 \leq \varepsilon.
\end{eqnarray}
for all $A$ with row norm $\|A_i\| \leq 2$.
\end{lemma}
\begin{proof}
We have that
\begin{eqnarray}
&&\norm{\nabla L(A)- \nabla L_T(A)}_2 \,\notag\\
&=& \norm{\Expect[\nabla l(x,y,A)(1-\mathbbm{1}_E)]} \,\notag\\
& \stackrel{(a)} \leq& \Expect[\norm{\nabla l(x,y,A) } \mathbbm{1}\{\norm{x} \geq R  \}  ] \,\notag\\
&=& \sum_{i\geq 0} \Expect[\norm{\nabla l(x,y,A) } \mathbbm{1}\{\norm{x} \in [2^i R, 2^{i+1}R]  \}  ] \,\notag\\
& \stackrel{(b)} \leq& \sum_{i \geq 0} K (2^{i+1}R)^p \prob{\norm{x} \geq 2^i R } \,\notag\\
& \leq& \sum_{i \geq 0} K (2^{i+1}R)^p e^{-2^i R} \,\notag\\
& \stackrel{(c)} \leq& \sum_{i \geq 0} e^{-2^{i-1}R} \,\notag\\
& =& \sum_{i \geq 0} \varepsilon^{C d 2^{i-1}} \,\notag\\
& \stackrel{(d)} \leq& \sum_{i \geq 0} \varepsilon/2^{i+1} = \varepsilon,
\end{eqnarray}
where $(a)$ follows from the Jensen's inequality, $(b)$ follows from \prettyref{assump:finiteassump}, $(c)$ follows from the fact that $K(2x)^p e^{-x} \leq e^{-x/2}$ for $x$ sufficiently large, and $(d)$ follows from choosing $C$ sufficiently large.
Similarly for $\norm{\nabla^2 L(A)- \nabla^2 L_T(A)}_2$.
\end{proof}

We are now ready to prove \prettyref{thm:finitesample}.
\begin{proof}
Let $A$ be such that norms of all the rows are less than $2$. Then we have from \prettyref{lmm:truncclose} that
\begin{eqnarray}
\norm{\nabla L(A)- \nabla L_T(A)}_2 \leq \varepsilon/4,\\
\norm{\nabla^2 L(A)- \nabla^2 L_T(A)}_2 \leq \tau_0/4.
\end{eqnarray}
Notice that the gradient and Hessian of $l(x,y,A)\mathbbm{1}_E$ are bounded $\tau=\mathrm{poly}(d,1/\varepsilon)$ for some fixed polynomial $\mathrm{poly}$. Hence using the uniform convergence of the sample gradients/Hessians to their population counterparts~\cite[Theorem E.3]{ge2017learning}, we have that
\begin{eqnarray}
\norm{\nabla L_T(A)- \nabla \hat{L}_T(A)}_2 \leq \varepsilon/6,\\
\norm{\nabla^2 L_T(A)- \nabla^2 \hat{L}_T(A)}_2 \leq \tau_0/6,
\end{eqnarray}
whenever $N \geq \mathrm{poly}(d,1/\varepsilon)$, with high probability. Moreover, from standard concentration inequalities (such as multivariate Chebyshev) it follows that
\begin{eqnarray}
&&\norm{\nabla \hat{L}(A)- \nabla \hat{L}_T(A)- (\nabla L(A)- \nabla L_T(A) )   }_2 \,\notag\\
&\leq& \varepsilon/6,\\
&&\norm{\nabla^2 \hat{L}(A)- \nabla^2 \hat{L}_T(A)- (\nabla^2 L(A)- \nabla^2 L_T(A) )}_2 \,\notag\\
&\leq& \tau_0/6,
\end{eqnarray}
with high probability, whenever $N \geq \mathrm{poly}(d,1/\varepsilon)$. Hence, we obtain that
\begin{align}
\norm{\nabla L(A)- \nabla \hat{L}(A)}_2 \leq \varepsilon/2,\\
\norm{\nabla^2 L(A)- \nabla^2 \hat{L}(A)}_2 \leq \tau_0/2.
\label{eq:final}
\end{align}
If $A$ is such that there exists a row $A_i$ with $\norm{A_i} \geq 2$, we have from~\cite[Lemma E.5]{ge2017learning} that $\inner{\nabla \hat{L}(A),A_i} \geq c \lambda \norm{A_i}^4$ for a small constant $c$ and thus $A$ cannot be a local minimum for $\hat{L}$. Hence all local minima of $\hat{L}$ must have $\norm{A_i} \leq 2$ and thus in view of \prettyref{eq:final} it follows that it also a $\varepsilon$-approximate local minima of $L$, or more concretely,
\begin{eqnarray}
\norm{\nabla L(A)} \leq \varepsilon, \quad \nabla^2 L(A) \succcurlyeq -\tau_0 I_d.
\end{eqnarray}
\end{proof}

\subsection{Landscape design for $k<d$}
\label{app:generalk}
In the setting where $k=d$ and the regressors $a_1^\ast,\ldots,a_d^\ast$ are linearly independent, our loss functions $L_4(\cdot)$ can modified in a straightforward manner to arrive at the loss function $F(\cdot)$ defined in Appendix C.2 of \cite{ge2017learning}. Hence we have the same landscape properties as that of Theorem B.1 of \cite{ge2017learning}. The proof is exactly similar to that of our \prettyref{thm:landscape}.

In a more geneal scenario where $k<d$ and the regressors $a_1^\ast,\ldots,a_d^\ast$ are linearly independent, it turns out that our loss function $L_4(\cdot)$ can also be transformed to obtain the loss $\calF(\cdot)$ in Appendix C.3 of \cite{ge2017learning} to arrive at Theorem C.1 of \cite{ge2017learning} in our setting. The proof is again similar.

\end{document}